\documentclass{article}
\usepackage{ijcai23}

\usepackage{times}
\usepackage{soul}
\usepackage{url}
\usepackage[hidelinks]{hyperref}
\usepackage[utf8]{inputenc}
\usepackage[small]{caption}
\usepackage{graphicx}
\usepackage{amsmath}
\usepackage{amssymb}
\usepackage{amsthm}
\usepackage{booktabs}
\usepackage{algorithm}
\usepackage{algorithmic}
\usepackage[switch]{lineno}

\DeclareMathOperator*{\argmin}{arg\,min}

\newtheorem{theorem}{Theorem}
\newtheorem{lemma}{Lemma}
\newtheorem{definition}{Definition}
\newtheorem{remark}{Remark}

\title{Graph-based Semi-supervised Local Clustering with Few Labeled Nodes}

\author{%Paper ID: 4376
Zhaiming Shen$^1$
\and
Ming-Jun Lai$^1$\And
Sheng Li$^{2}$ 
\affiliations
$^1$University of Georgia, Athens, GA, USA\\
$^2$University of Virginia, Charlottesville, VA, USA
\emails
\{zhaiming.shen, mjlai\}@uga.edu,
shengli@virginia.edu
}
\begin{document}

\maketitle

\begin{abstract}
Local clustering aims at extracting a local structure inside a graph without the necessity of knowing the entire graph structure. As the local structure is usually small in size compared to the entire graph, one can think of it as a compressive sensing problem where the indices of target cluster can be thought as a sparse solution to a linear system.
In this paper, we apply this idea based on two pioneering works under the same framework and propose a new semi-supervised local clustering approach using only few labeled nodes. Our approach improves the existing works by making the initial cut to be the entire graph and hence overcomes a major limitation of the existing works, which is the low quality of initial cut. Extensive experimental results on various datasets demonstrate the effectiveness of our approach.
\end{abstract}

\section{Introduction}
Being able to learn from data by investigating its underlying pattern, and separate data into different groups or clusters based on their latent similarity and differences is one of the main interests in machine learning and artificial intelligence. There are many clustering phenomena across disciplines such as social science, health science, and engineering.
Through the past few decades, traditional clustering problem has been studied a lot and many algorithms have been developed, such as $k$-means clustering \cite{MacQueen1967}, hierarchical clustering \cite{Nielsen2016}, and density based clustering \cite{Ester1996}.
For graph structured data or the data which can be converted into a graph structure by applying some techniques (e.g. K-NN auxiliary graph), it is natural to consider the task as a graph clustering problem.

Traditional graph clustering problem assumes the underlying data structure as a graph where data points are the nodes and the connections between data points are the edges. It assigns each node into a unique cluster, assuming there are no multi-class assignments. For nodes with high connection density, they are considered in the same cluster, and for nodes with low connection density, they are considered in different clusters. Since the task is to learn the clustering patterns by investigating the underlying graph structure, it is an unsupervised learning task. Many unsupervised graph clustering algorithms have been developed through decades. For example, spectral clustering \cite{Ng2001}, which is based on the eigen-decomposition of Laplacian matrices of either weighted or unweighted graphs. Based on this, many variants of spectral clustering algorithms have been proposed, such as \cite{Zelnik2004} and \cite{Huang2012CVPR}.  Another category is the graph partition based method such as finding the optimal cut \cite{Dhillon2004,Ding2001}. It is worthy noting that spectral clustering and graph partition have the same essence, see \cite{Luxburg2007}. Spectral clustering has become one of the popular modern clustering algorithms since it enjoys the advantage of exploring the intrinsic data structures. It is simple to implement, and it often outperforms the traditional algorithm such as $k$-means. However, one of the main drawbacks of spectral clustering is its high computational cost, so it is usually not applicable to large datasets. Meanwhile, the spectral clustering method does not perform well on the auxiliary graphs which are generated from certain shapes of numerical data, e.g., elongated band shape data and moon shape data. In addition, many other clustering methods have been developed, such as the low rank and sparse representations based methods \cite{Liu2013,Huang2015}, deep embedding based methods \cite{Xie2016}, and graph neural network based methods \cite{Hui2020,Tsitsulin2020}. Besides the unsupervised way, some semi-supervised graph clustering methods have also been proposed \cite{Kulis2009,Kang2021,Ren2019}.  

These clustering algorithms, whether unsupervised or semi-supervised, are all global clustering algorithms, which means that the algorithms output all the clusters simultaneously. However, it is often to people's interests in only finding a single target cluster which contains the given labels, without worried too much about how the remaining part of graph will be clustered. Such an idea is very useful in detecting small-scale structure in a large-scale graph. This type of problem is referred to as \emph{local clustering} or \emph{local cluster extraction}. Most of the current local clustering algorithms aim at finding the best cut from the graph, for example, \cite{Veldt2019,Fountoulakis2020,Orecchia2014}. 
% There is also idea of using graph sparsification for local clustering \cite{Satuluri2011} and perspective on high order local clustering \cite{Yin2017}, \cite{Zhou2021}.
It is worth pointing out that \cite{Fountoulakis2018} have kindly put several methods of local graph clustering into software, including both the spectral methods \cite{Andersen2006,Fountoulakis2019} and flow-based methods \cite{Lang2004,Veldt2016,Wang2017}. 
%These include three spectral methods, the approximate PageRank \cite{Andersen2006}, PageRank Nibble \cite{Andersen2006}, $\ell_1$ regularized PageRank \cite{Fountoulakis2019}, and four flow methods, the Max-flow Quotient-cut Improvement \cite{Lang2004}, FlowImprove \cite{Andersen2008}, SimpleLocal \cite{Veldt2016} and Capacity Releasing Diffusion \cite{Wang2017}.
More recently, new approaches for making the cut of graph based on the idea of compressive sensing are proposed in \cite{LaiMckenzie2020} and \cite{LaiShen2022}, where they took a novel perspective by considering the way of finding the optimal cut as an improvement from an initial cut via finding a sparse solution to a linear system. However, the performances of their approaches will heavily depend on the quality of initial cut.

In this paper, based on the idea of compressive sensing, we propose a semi-supervised local clustering method using only few labeled nodes from the target cluster, with theoretical guarantees. %It improves the existing work
%Algorithm \ref{alg2} based on the idea of \cite{LaiMckenzie2020} and \cite{LaiShen2022}. 
Our approach improves the existing works by making the initial cut to be the entire graph and hence overcomes the issue that missing vertices of the target cluster from the initial cut are not recoverable in the later stage. Extensive experiments are conducted on various benchmark datasets to show our approach outperforms its counterparts \cite{LaiMckenzie2020,LaiShen2022}. Results also show that our approach is favorable than many other state-of-the-art semi-supervised clustering algorithms. %Furthermore, it is worth pointing out that our algorithm achieves the theoretical threshold for stochastic block model in \cite{?} and it also improves the results in \cite{?} for the political blog network data \cite{?}.

% The subsequent sections of the paper are structured as follows. In Section \ref{secPrelim}, we introduce some necessary notations and preliminary concepts, the idea of compressive sensing, and then formalize the problem. In Section \ref{secAlg}, we motivate the ideas of our approach, propose the main algorithm, and show its asymptotic correctness of finding the true cluster. In Section \ref{secExp}, we evaluate our algorithm on various datasets and compare its performance with baselines. In Section \ref{secConclusion}, we draw conclusions and discuss potential future research directions.

\section{Preliminaries} \label{secPrelim}
\subsection{Graph Notations and Concepts}
We adopt the standard notations for graph $G=(V, E)$, where $V$ is the set of all vertices and $E$ is the set of all edges.
In the case that the size of graph equals to $n$, we identify $V=\{1,2,\cdots, n\}=[n]$. For a graph $G$ with $k$ non-overlapping underlying clusters $C_1, C_2, \cdots, C_k$, we use $n_i$ to indicate the size of $C_i$ where $i=1,2,\cdots,k$. Without loss of generality, let us assume $n_1\leq n_2\leq \cdots \leq n_k$. Furthermore, we use matrix $A$ to denote the adjacency matrix of graph $G$, and use $D$ to denote the diagonal matrix where each diagonal entry in $D$ is the degree of the corresponding vertex. In addition, we define the notion called graph Laplacian.

\begin{definition}
The unnormalized graph Laplacian of graph G is defined as $L=D-A$. The symmetric graph Laplacian of graph G is defined as $L_{sym}: =I-D^{-1/2}AD^{-1/2}$ and the random walk graph Laplacian is defined as  $L_{rw}: = I-D^{-1}A$.
\end{definition}
For the scope of our problem, we will only focus on $L_{rw}$ for the rest of discussion and we will use $L$ to denote $L_{rw}$ for the concise of notation. Recall the following fundamental result from spectral graph theory. We omit the proof by referring to \cite{Chung1997} and \cite{Luxburg2007}.
\begin{lemma} \label{kernelthm}
	Let G be an undirected graph with non-negative weights. The multiplicity $k$ of the eigenvalue zero of $L$ equals to the number of connected components  $C_1, C_2, \cdots, C_k$ in $G$, and the indicator vectors $\textbf{1}_{C_1}, \cdots, \textbf{1}_{C_k}\in\mathbb{R}^n$ on these components span the kernel of $L$.
\end{lemma}
For a graph $G$ with underlying structure which separates vertices into different clusters, we can write $G=G^{in}\cup G^{out}$, where $G^{in}=(V, E^{in})$ and $G^{out}=(V, E^{out})$. Here $E^{in}$ is the set of all intra-connection edges within the same cluster, $E^{out}$ is the set of all inter-connection edges between different clusters. We use $A^{in}$ and $A^{out}$ to denote the adjacency matrices associated with $G^{in}$ and $G^{out}$ respectively, and use $L^{in}$ and $L^{out}$ to denote the Laplacian matrices associated with $G^{in}$ and $G^{out}$ respectively.
From their definitions, we can easily see $L^{in}$ is in a block diagonal form if the vertices are sorted according to their memberships.
It is worthwhile to point out that $A=A^{in}+A^{out}$ but $L\neq L^{in}+L^{out}$ in general.

\begin{remark} \label{remarkNotation}
    We introduce these notations in order for the convenience of our discussion in the later sections. Note that in reality we will have no assurance about which cluster each individual vertex belongs to, so we have no access to $A^{in}$ and $L^{in}$. What we have access to are $A$ and $L$. 
\end{remark}

Furthermore, for a set $S$, we use $|S|$ to denote its size. For a matrix $M$ or vector $\mathbf{v}$, we use $|M|$ or $|\mathbf{v}|$ to denote the matrix or vector where each of its entry is replaced by the absolute value. For a matrix $M$ and a set $S\subset V$, we use $M_S$ to denote the submatrix of $M$ where the columns of $M_S$ consist of only the indices in $S$. For convenience, we summarize the notations being used throughout this paper in Table \ref{Notation}. 

\begin{table}[t]
    \centering
    \resizebox{\columnwidth}{!}{
    \begin{tabular}{ll}
        \toprule
         Symbols & Interpretations \\
         \midrule
         $G$ & A general graph of interest \\
         $|G|$ & Size of G \\
        $V$ & Set of vertices of graph G \\
        $|V|$ & Size of $V$ \\
        $C_1$ & Target Cluster\\
        $\Gamma$ & Set of Seeds \\
        $T$ & Removal set from $V$\\
        $E$ & Set of edges of graph G \\
        $E^{in}$ & Subset of $E$ which consists only intra-connection edges \\
        $E^{out}$ & Subset of $E$ which consists only inter-connection edges \\
        $G^{in}$ & Subgraph of $G$ on $V$ with edge set $E^{in}$ \\
        $G^{out}$ & Subgraph of $G$ on $V$ with edge set $E^{out}$ \\
        $A$ & Adjacency matrix of graph $G$ \\
        $A^{in}$ & Adjacency matrix of graph $G^{in}$ \\
        $A^{out}$ & Adjacency matrix of graph $G^{out}$ \\
        $L$ & Random walk graph Laplacian of $G$ \\
        $L^{in}$ & Random walk graph Laplacian of $G^{in}$ \\
        $L^{out}$ & Random walk graph Laplacian of $G^{out}$  \\
        $L_{C}$ & submatrix of $L$ with column indices $C\subset V$ \\
        $L^{in}_{C}$ & submatrix of $L^{in}$ with column indices $C\subset V$ \\
        $|M|$  & Entrywised absolute value operation on matrix $M$ \\ 
        $\|M\|_2$ & $\|\cdot\|_2$ norm of matrix $M$ \\
        $|\mathbf{v}|$ & Entrywised absolute value operation on vector $\mathbf{v}$ \\
        $\|\mathbf{v}\|_2$ & $\|\cdot\|_2$ norm of vector $\mathbf{v}$.
         \\
         $\mathbf{1}_{C}$  & Indicator vector on subset $C\subset V$ \\
         $\triangle$ & Set symmetric difference \\
         $Ker$ & Kernel of the linear map induced by a matrix \\
         $Span$ & Spanning set of a set of vectors \\
         $\mathcal{L}_s(\mathbf{v})$ & $\{i\in[n]: \text{$v_i$ among $s$ largest-in-magnitude entries in $\mathbf{v}$}\}$ \\
       \bottomrule
    \end{tabular}
    }
    \caption{Table of Notations}
    \label{Notation}
\end{table}

%\subsection{Graph Model Assumptions}

\subsection{Compressive Sensing}
Recall that $\|\cdot\|_0$ counts the number of nonzero components in a vector. The idea of compressive sensing comes from solving the optimization problem:
\begin{equation} \label{CS1}
    \min \|\mathbf{x}\|_0 \quad s.t. \quad \|\Phi\mathbf{x}-\mathbf{y}\|_2\leq\epsilon,
\end{equation} 
where $\Phi\in\mathbb{R}^{m\times n}$ is called sensing matrix,  $\mathbf{y}\in\mathbb{R}^n$ is called measurement vector. The goal is to recover the sparse solution $\mathbf{x}\in\mathbb{R}^n$ under some constraints. It can be reformulated as solving:
\begin{equation} \label{CS2}
    \argmin \|\Phi\mathbf{x} - \mathbf{y}\|_2 \quad s.t. \quad \|\mathbf{x}\|_0\leq s.
\end{equation}
Its idea was first introduced by Dohono \cite{Donoho2006} and Candès, Romberg, Tao \cite{Candes2006}. Since then, many algorithms have been developed to solve (\ref{CS1}) or (\ref{CS2}), including the greedy based approaches such as orthogonal matching pursuit (OMP) \cite{Tropp2004} and its variants, quasi-orthogonal matching pursuit (QOMP) \cite{Feng2021}, thresholding based approaches such as iterative hard thresholding \cite{Blumensath2009}, compressive sensing matching pursuit (CoSAMP) \cite{Needell2009}, and subspace pursuit \cite{Dai2009}, etc,. Note that (\ref{CS1}) is NP-hard because of the appearance of zero norm. Therefore it is sometimes convenient to solve its $\ell_1$ convex relaxation:
\begin{equation} \label{CS3}
    \min \|\mathbf{x}\|_1 \quad s.t. \quad \|\Phi\mathbf{x}-\mathbf{y}\|_2\leq\epsilon.
\end{equation} 
Algorithms such as LASSO \cite{Tibshirani1996}, CVX \cite{Grant2008}, and reweighted $\ell_1$-minimization \cite{Candes2008} fall into this category. We do not analyze further here. The monograph \cite{LaiWang2021} gives a comprehensive summary of these algorithms. 

It is worthwhile to mention one of the key concepts in compressive sensing, Restricted Isometry Property (RIP), which guarantees a good recovery of the solution to (\ref{CS1}).
\begin{definition}
Let $\Phi\in\mathbb{R}^{m\times n}$, $1\leq s\leq n$ be an integer. Suppose there exists a constant $\delta_s\in (0,1)$ such that
\begin{equation} \label{RIP}
    (1-\delta_s)\|\mathbf{x}\|_2^2\leq\|\Phi\mathbf{x}\|_2^2\leq (1+\delta_s)\|\mathbf{x}\|_2^2
\end{equation}
for all $\mathbf{x}\in\mathbb{R}^n$ with $\|\mathbf{x}\|_0\leq s$. Then the matrix $\Phi$ is said to have the Restricted Isometry Property (RIP). The smallest constant $\delta_s$ which makes (\ref{RIP}) hold is called the Restricted Isometry Constant (RIC).
\end{definition}
% \begin{remark}
% The RIP condition is a sufficient condition for the matrix to have possible sparse vector recovery, but it is very hard to verify. There are other conditions such as mutual incoherence \cite{Donoho2003} which is simple and sufficient, but not sharp. Null space property \cite{Cohen2009}, which is necessary and sufficient, but also very difficult to verify.
% \end{remark}
Another very important aspect which makes compressive sensing very useful is its robustness to noise. Suppose we try to solve the linear system $\mathbf{y}=\Phi \mathbf{x}$ given the measurement $\mathbf{y}$ and sensing matrix $\Phi$. It is possible that we only have access to a noise version of $\Phi$, say $\Tilde{\Phi}=\Phi+\epsilon_{1}$, and also only have access to a noise version of $\mathbf{y}$, say $\mathbf{\Tilde{y}}=\mathbf{y}+\mathbf{\epsilon}_2$. Therefore, instead of solving $\mathbf{y}=\Phi \mathbf{x}$, what we solve in reality is $\mathbf{\Tilde{y}}=\Tilde{\Phi}\mathbf{\Tilde{x}}$. However, if $\epsilon_1, \epsilon_2$ are both small in some sense, and the sensing matrix $\Phi$ satisfies certain conditions, then we will have $\mathbf{\Tilde{x}}\approx\mathbf{x}$. There are plenty of ways to solve $\mathbf{\Tilde{x}}$ given $\mathbf{\Tilde{\Phi}}$ and $\mathbf{\Tilde{y}}$, what we will be focusing on is subspace pursuit \cite{Dai2009}. Theorem 2.5 in \cite{LaiMckenzie2020} and Corollary 1 in \cite{Li2016} gives a result about how close $\mathbf{\Tilde{x}}$ and $\mathbf{x}$ can be based on the conditions of $\Tilde{\Phi}$, $\Phi$, $\mathbf{\Tilde{y}}$, $\mathbf{y}$, which we will apply later in our theoretical analysis part.

\subsection{Problem Statement}
%\paragraph{Problem Statement}
For convenience, let us assume the target cluster is the first cluster $C_1$ for the rest of discussion. Now let us formally state the local clustering task that we are interested in:

Suppose $G=(V,E)$ is a graph with underlying cluster $C_1, \cdots, C_k$ where $V=\cup_{i=1}^n C_i$, $C_i\cap C_j=\emptyset$ for $1\leq i, j\leq k$, $i\neq j$. Given a set of labeled vertices $\Gamma\subset C_1$, which we call them seeds, assuming the size of $\Gamma$ is small relative to the size of $C_1$. The goal is to extract all the vertices in the target cluster $C_1$.

\section{Local Clustering via Compressive Sensing}  \label{secAlg}

The local clustering task can be considered as a compressive sensing problem in the following way. Suppose the vertices have been sorted according to their memberships, i.e., the first $n_1$ rows and columns in $L^{in}$ corresponds to all the vertices in $C_1$, the last $n_k$ rows and columns corresponds to all the vertices in $C_k$, etc,.  

Let $L^{in}_{-1}$ be the matrix obtained from $L^{in}$ by deleting the first column from $C_1$. For this particular graph, all the clusters have size three, the symbol $*$ equals to $-1/2$, and all the other entries in the off-diagonal blocks equal to zero. 
\begin{equation}
L^{in}_{-1} = \left( 
\begin{array}{ccccccccc}
* & * &  &  & & & & &   \\
1 & * &  &  & & & & &    \\
* & 1 &  &  & & & & &   \\
&  & 1 & * & * & & & &  \\
&  & * & 1 & * & & & &    \\
&  & * & * & 1 & & & &     \\
&  &   &   &   & \ddots & & &  \\
&  &  &  &  &  & 1 & * & *    \\
&  &  &  &  &  & * & 1 & *    \\
&  &  &  &  &  & * & * & 1    \\
\end{array}
\right)
\end{equation}
Let $\mathbf{y}^{in}$ be the row sum vector of $L^{in}_{-1}$.
Then the desired solution to the compressive sensing problem
\begin{equation}
    \min\|\mathbf{x}\|_0 \quad s.t. \quad L^{in}_{-1}\mathbf{x}=\mathbf{y}^{in}
\end{equation}
is $\mathbf{x^*}=(1,1,0,\cdots,0)'$. The significance of this formulation is that the nonzero components in $\mathbf{x^*}$ correspond to the indices of vertices which belong to the target cluster $C_1$.  This gives us the intuitive idea of how to apply compressive sensing for solving local clustering problem. 

As noted in Remark \ref{remarkNotation}, we usually do not have access to $L^{in}$ or $L^{in}_{-1}$, what we do have access to are $L$ and $L_{-1}$. We can relax the exact equality condition to approximately equal to, so the problem becomes
\begin{equation} \label{noisynegativeone}
    \min\|\mathbf{x}\|_0 \quad s.t. \quad L_{-1}\mathbf{x}\approx\mathbf{y},
\end{equation}
where $\mathbf{y}$ is the row sum vector of $L_{-1}$.
Let $\mathbf{x^\#}$ be the solution to (\ref{noisynegativeone}). 
Suppose the graph has a good underlying clusters structure, in other words, the entries in the off-diagonal block of $L_{-1}$ have very small magnitude, i.e., $L_{-1}\approx L_{-1}^{in}$. Then we should have $\mathbf{y}\approx\mathbf{y}^{in}$, and hence the difference between $\mathbf{x^\#}$ and $\mathbf{x^*}$ should be small in certain sense. We can then use some cutoff number $R\in (0,1)$ to separate the coordinates of $\mathbf{x^\#}$ and therefore extract the target cluster from the entire graph.

\subsection{Main Algorithm} 
In general, we can remove more than just one column. That is, we remove a set $T\subset V$ in a somewhat smart way, with the hope that $T\subset C_1$, and then we solve
\begin{equation} \label{noisyminusT}
    \min\|\mathbf{x}\|_0 \quad s.t. \quad \|L_{V\setminus T}\mathbf{x}-\mathbf{y}\|_2\leq\epsilon.
\end{equation}
Or equivalently, we solve
\begin{equation}
\label{outalglsqremoveT}
\argmin_{\mathbf{x}  \in \mathbb{R}^{|V|-|T|}} \{\|L_{V\setminus T}\mathbf{x}- \mathbf{y}\|_2: \|\mathbf{x}\|_0\leq s\}
\end{equation}
where vector $\mathbf{y}$ is the row sum vector of $L_{V\setminus T}$ and $s$ is the sparsity constraint.

Naively, if the size of $\Gamma$ is not too small, then we can just choose $T=\Gamma$. However, for the scope of our problem, the size of $\Gamma$ is assumed to be small relative to the size of $C_1$, therefore this choice does not work well in practice.
Instead, we select $T$ based on a heuristic criterion (as described in step 4) on a candidate set $\Omega$ which is obtained from a random walk originates from $\Gamma$.  We also find that the size of $T$ does not matter too much based on our exploration in the experiments. The idea is summarized in Algorithm \ref{alg2} as CS-LCE. We give a more detailed explanation about several aspects of the algorithm in Remark \ref{remarkOmega} and Remark \ref{remarkSP}. More generally, we can apply CS-LCE iteratively to extract all the clusters one at a time. 
%The way we select set $T$ is as follows. Suppose the Adjancey matrix $A$ of the graph is given, we start from a set of vertices with known labels, we call them seed vertices, or seeds. Then we perform a random walk originates from the seeds up to a certain depth, then choose the set of column indices $T$ we want to remove from $V$ according to the probability of being visited in the random walk. 

\begin{algorithm}[tb]
    \caption{Compressive Sensing of Local Cluster Extraction (CS-LCE)}
    \label{alg2}
    \textbf{Input}: Adjacency matrix $A$, and a small set of seeds $\Gamma\subset C_1$ \\
    \textbf{Parameter}: Estimated size $\hat{n}_1\approx |C_1|$, random walk threshold parameter $\epsilon\in (0,1)$, random walk depth $t\in\mathbb{Z}^{+}$, sparsity parameter $\gamma\in [0.1, 0.5]$, rejection parameter $R\in [0.1, 0.9]$\\
    \textbf{Output}: The target cluster $C_1$
    \begin{algorithmic}[1] %[1] enables line numbers
        \STATE Compute $P=AD^{-1}$,  $\mathbf{v}^{0}=D\mathbf{1}_{\Gamma}$, and $L=I-D^{-1}A$.
        \STATE Compute $\mathbf{v}^{(t)}=P^t\mathbf{v}^{(0)}$.
        \STATE Define $\Omega={\mathcal{L}}_{(1+\epsilon)\hat{n}_1}(\mathbf{v}^{(t)})$. 
        \STATE Let $T$ be the set of column indices of $\gamma\cdot|\Omega|$ smallest components of the vector $|L_{\Omega}^{\top}|\cdot|L\mathbf{1}_{\Omega}|$.
        \STATE Set $\mathbf{y}:=L\mathbf{1}_{V\setminus T}$. Let $\mathbf{x}^\#$ be the solution to
        \begin{equation} 
        \label{inalglsqremoveT}
        \argmin_{\mathbf{x}  \in \mathbb{R}^{|V|-|T|}} \{\|L_{V\setminus T}\mathbf{x}- \mathbf{y}\|_2: \|\mathbf{x}\|_0\leq (1-\gamma)\hat{n}_1\}
        \end{equation}
        obtained by using $O(\log n)$ iterations of \emph{Subspace Pursuit} \cite{Dai2009}.
        \STATE Let $W^{\#} = \{i: \mathbf{x}_i^{\#}>R\}$  . 
        \STATE \textbf{return} $C^{\#}_1=W^{\#}\cup T$. 
    \end{algorithmic}
\end{algorithm}

We would like to point out the major differences between CS-LCE with its counterparts CP+RWT in \cite{LaiMckenzie2020} and LSC in \cite{LaiShen2022}. The key difference is that the latter two methods only be able to extract target cluster from the initial cut $\Omega$, since it is assumed that $C_1\subset \Omega$ in these two methods before extracting all the vertices in $C_1$, and once $\Omega$ fails to contain any vertex in $C_1$, there is no chance for CP+RWT or LSC to recover those vertices in the later stage. However, such an assumption is not needed in CS-LCE. Since the sensing matrix in CS-LCE is associated with all the vertices corresponding to $V\setminus T$, it is very probable for CS-LCE to still be able to find the vertices which are in $C_1$ but not in $\Omega$. 

\begin{remark} \label{remarkOmega}
The purpose of $\Omega$ is solely for obtaining the set $T$, and the vector $\mathbf{y}$ is computed by adding up all the columns with indices in the set $V\setminus T$. This is another key difference between CS-LCE and CP+RWT \cite{LaiMckenzie2020} and LSC \cite{LaiShen2022}, whereas the latter two methods directly use $\Omega$ to obtain $\mathbf{y}$. 
\end{remark}

\begin{remark} \label{remarkSP}
The rationale for choosing an iterative approach such as Subspace Pursuit over other sophisticated optimization algorithms for solving (\ref{inalglsqremoveT}) comes from the nature of our task. Since the task is clustering, all we need is a relative good estimated solution instead of the exact solution, then we can use a cutoff number $R$ in Algorithm \ref{alg2} to separate the aimed cluster from the remaining of the graph. Due to the nature of an iterative approach, the convergence is usually fast at the beginning and slow in the end, so we can stop early in the iteration to save the computational cost once the estimated solution is roughly ``close enough'' to the true solution. 
\end{remark}

% \begin{remark} \label{remarkremoveT}
% The purpose of removing a small subset $T$ from $V$ is to make our problem well-posed. We find heuristically the way to obtain such $T$ based on the criterion in Step 4 of the algorithm works well. We also find that the size of $T$ does not matter too much based on our exploration during the experiments. 
% \end{remark}

% \begin{algorithm}[ht]
% \caption{\textbf{Compressive Sensing of Iterative Local Cluster Extraction (CS-ILCE)}}
% \label{alg3}
% \begin{algorithmic}
% \Require 
% Adjacency matrix $A$, random walk threshold parameter $\epsilon\in (0,1)$, least squares parameter $\gamma\in (0,0.8)$, rejection 
% parameter $R\in[0,1)$, depth of random walk $t\in\mathbb{Z}^{+}$. Seed vertices for each cluster $\Gamma_i\subset C_i$, estimated 
% size $\hat{n}_i\approx |C_i|$ for $i=1,\cdots k$. 
% \begin{itemize}
% \item{} \textbf{for} $i= 1,\cdots, k$
% \item{} \quad Let $C_i^{\#}$ be the output of \textbf{Algorithm 1}.
% \item{} \quad  Let $G^{(i)}$ be the subgraph spanned by $C_i^{\#}$.
% \item{} \quad Updates $G\leftarrow G\setminus G^{(i)}$.
% \item{} \textbf{end}
% \end{itemize}
% \Ensure $C_1^{\#}, \cdots, C_k^{\#}$. 
% \end{algorithmic}
% \end{algorithm}

\subsection{Theoretical Analysis}
For convenience, let us fix $\gamma=0.4$ for the rest of discussion.
We want to make sure the output $C_1^\#$ from Algorithm \ref{alg2} is as close to the true cluster $C_1$ as possible. In order to investigate more towards this aspect, let us use $\mathbf{x}^*$ to denote the solution to the unperturbed problem:
\begin{equation} \label{noiseless}
\mathbf{x}^*:=\argmin_{\mathbf{x}  \in \mathbb{R}^{|V|-|T|}} \{\|L^{in}_{V\setminus T}\mathbf{x}- \mathbf{y}^{in}\|_2: \|\mathbf{x}\|_0\leq 0.6n_1\}
\end{equation}
where $\mathbf{y}^{in}=L^{in}\mathbf{1}_{V\setminus T}$. Let $\mathbf{x}^\#$ be the solution to (\ref{inalglsqremoveT}), the perturbed problem, with $\gamma=0.4$.

Let us first establish the correctness of having $\mathbf{x}^*$ equals to an indicator vector as the solution to (\ref{noiseless}), and then conclude that $\mathbf{x}^\#\approx\mathbf{x}^*$ if $L\approx L^{in}$ in a certain sense. Once this is established, we will be able to conclude $C_1^\#\approx C_1$. These results are summarized in the following as a series of theorems and lemma.
\begin{theorem}
Suppose $T\subset C_1$. Then $\mathbf{x}^*=\mathbf{1}_{C_1\setminus T}\in\mathbb{R}^{|V|-|T|}$ is the unique solution to (\ref{noiseless}).
\end{theorem}
\begin{proof}
Note that for $\mathbf{y}^{in}=L^{in}\mathbf{1}_{V\setminus T}$, we can rewrite it as $\mathbf{y}^{in}=L^{in}_{V\setminus T}\mathbf{1}$ where $\mathbf{1}\in\mathbb{R}^{|V|-|T|}$.
It is straightforward to check $\mathbf{x}^*=\mathbf{1}_{C_1\setminus T}$ is a solution to (\ref{noiseless}). The rest is to show it is unique. 

Suppose otherwise, then since $L^{in}_{V\setminus T}\mathbf{1}_{C_1\setminus T}=\mathbf{y}^{in}$, we want to find $\mathbf{x}\in\mathbb{R}^{|V|-|T|}$ and $\mathbf{x}\neq\mathbf{1}_{C_1\setminus T}$ such that $L^{in}_{V\setminus T}(\mathbf{x}-\mathbf{1})=\mathbf{0}$. Without loss of generality, let us assume the columns of $L$ are permuted such that it is in the block diagonal form, i.e.,
\begin{equation*}
L_{V\setminus T}^{in} = \left( \begin{array}{cccc}
L^{in}_{C_1\setminus T} &  &  &  \\
& L^{in}_{C_2}  &   & \\
& & \ddots & \\
&  &   &  L^{in}_{C_n}  \\
\end{array} \right).
\end{equation*}

Let us now show that $L^{in}_{C_1\setminus T}$ is of full column rank, i.e., the columns of $L^{in}_{C_1\setminus T}$ is linearly independent. 
We first observe the following fact.
By Lemma \ref{kernelthm}, each of $L^{in}_{C_i}$ has $\lambda=0$ as an eigenvalue with multiplicity one, and the corresponding eigenspace is spanned by $\mathbf{1}_{C_i}$.
Now suppose by contradiction that the columns of $L^{in}_{C_1\setminus T}$ are linearly dependent, so there exists $\mathbf{v}\neq\mathbf{0}$ such that $L^{in}_{C_1\setminus T}\mathbf{v}=\mathbf{0}$, or $L^{in}_{C_1\setminus T}\mathbf{v} + L^{in}_T\cdot\mathbf{0}=\mathbf{0}$. This means that $\mathbf{u}=(\mathbf{v},\mathbf{0})$ is an eigenvector associated to eigenvalue zero, which contradicts the fact that the eigenspace is spanned by $\mathbf{1}_{C_i}$. Therefore $L^{in}_{C_1\setminus T}$ is of full column rank. 

Since $L^{in}_{C_1\setminus T}$ is of full column rank, and $Ker(L^{in}_{C_i})=Span\{\mathbf{1}_{C_i}\}$ for $i\geq 2$. We conclude that $\mathbf{x}-\mathbf{1}\in Ker(L^{in}_{V\setminus T})=Span\{\mathbf{1}_{C_2}, \cdots, \mathbf{1}_{C_n}\}$. Therefore in order to satisfy $\|\mathbf{x}\|_0\leq 0.6n_1$, it is easy to see $\mathbf{x}=\mathbf{1}-\mathbf{1}_{C_2}-\mathbf{1}_{C_3}-\cdots-\mathbf{1}_{C_k}=\mathbf{1}_{C_1\setminus T}$, which results in a contradiction by our assumption. 
\end{proof}

The next theorem shows that $\mathbf{x}^{*}$ and $\mathbf{x}^{\#}$ are close to each other if $L$ and $L^{in}$ are close.

\begin{theorem} \label{IndAna}
Let $M:=L-L^{in}$. Suppose $T\subset C_1$, $\|M\|_2=o(n^{-1/2})$ and $\delta_{1.8n_1}(L)=o(1)$. Then
\begin{equation}
\frac{\|\mathbf{x}^\#-\mathbf{x}^*\|_2}{\|\mathbf{x}^*\|_2}= o(1).
\end{equation}
\end{theorem}

\begin{proof}
Recall that $\mathbf{x}^\#$ is the output to (\ref{inalglsqremoveT}) after $O(\log n)$ iterations of \emph{Subspace Pursuit}.
By our assumption on $M$, we have
\begin{align*}
\|\mathbf{y}-\mathbf{y}^{in}\|_2&=\|L\mathbf{1}_{V\setminus T}-L^{in}\mathbf{1}_{V\setminus T}\|_2=\|(L-L^{in})\mathbf{1}_{V\setminus T}\|_2 \\
&\leq \|M\|_2\|\mathbf{1}_{V\setminus T}\|_2\leq o(n^{-1/2})\cdot \sqrt{n}=o(1).
\end{align*} 
Then applying Theorem 2.5 in \cite{LaiMckenzie2020}, we get the desired result.
\end{proof}

\begin{lemma} \label{indicatornorm}
Consider $K\subset [n]$, any $\mathbf{v}\in\mathbb{R}^{n}$, and $W^{\#}=\{i:\mathbf{v}_{i}>R \}$. If
$\|\mathbf{1}_{K}-\mathbf{v}\|_2\leq D$, then $|K\triangle W^{\#}|\leq \frac{D^2}{\min\{(1-R)^2, R^2\}}$.
\end{lemma}

\begin{proof}
Let $U^\#=[n]\setminus W^\#$, we can write $\mathbf{v}=\mathbf{v}_{U^\#}+\mathbf{v}_{W^\#}$ where $\mathbf{v}_{U^\#}$ and $\mathbf{v}_{W^\#}$ are the components of $\mathbf{v}$ supported on $U^\#$ and $W^\#$ respectively. Then we have
\begin{align*}
    \|\mathbf{1}_K-\mathbf{v}\|_2^2&=\|\mathbf{1}_K-\mathbf{v}_{U^\#}-\mathbf{v}_{W^\#}\|_2^2 \\
    & = \|\mathbf{1}_{K\setminus W^{\#}}-\mathbf{v}_{U^\#}\|_2^2+\|\mathbf{v}_{W^\#\setminus T}\|_2^2  \\
    &+\|\mathbf{1}_{K\cap W^{\#}}-\mathbf{v}_{K\cap W^{\#}}\|_2^2 \\
    &\geq \|\mathbf{1}_{K\setminus W^{\#}}-\mathbf{v}_{U^\#}\|_2^2+\|\mathbf{v}_{W^\#\setminus T}\|_2^2 \\
    &\geq (1-R)^2\cdot|K\setminus W^{\#}|+R^2\cdot|W^{\#}\setminus K| \\
    &\geq \min\{(1-R)^2, R^2\}(|K\setminus W^\#|+|W^\#\setminus K|) \\
    &=\min\{(1-R)^2, R^2\}|K\triangle W^\#|.
\end{align*}
Therefore $\|\mathbf{1}_{K}-\mathbf{v}\|_2\leq D$ implies $|K\triangle W^\#|\leq \frac{D^2}{\min\{(1-R)^2, R^2\}}$ as desired.
\end{proof}

%\begin{algorithm}[ht]
%\caption{\textbf{RandomWalkThreshold} }
%\begin{algorithmic}
%\Require 
% Adjacency matrix $A$, a random walk threshold parameter $\epsilon\in (0,1)$, a set of seed vertices $\Gamma\subset C_1$, estimated size $\hat{n}_1\approx |C_1|$, and depth of random walk $t\in\mathbb{Z}^{+}$.
%\begin{itemize}
%\item{} Compute $P=AD^{-1}$ and $\mathbf{v}^{0}=D\mathbf{1}_{\Gamma}$.
%\item{} Compute $\mathbf{v}^{(t)}=P^t\mathbf{v}^{(0)}$.
%\item{} Define $\Omega={\mathcal{L}}_{(1+\epsilon)\hat{n}_1}(\mathbf{v}^{(t)})$.  
%\end{itemize}
%\Ensure $\Omega=\Omega\cup\Gamma$. 
%\end{algorithmic}
%\end{algorithm}

%\begin{algorithm}[ht]
%\caption{\textbf{XXXXX} }
%\label{alg3}
%\begin{algorithmic}
%\Require 
%Adjacency matrix $A$, a random walk threshold parameter $\epsilon\in (0,1)$, a set of seed vertices $\Gamma\subset C_1$, estimated size $\hat{n}_1\approx |C_1|$, depth of random walk $t\in\mathbb{Z}^{+}$, least squares parameter $\gamma\in (0,0.8)$, and rejection parameter $R\in[0,1)$. 
%\begin{itemize}
%\item{} $\Omega=$ %\textbf{RandomWalkThreshold($A$, $\Gamma$, $\hat{n}_1$, $\epsilon$, $t$)}.
%\item{} $C_1^{\#}=$ \textbf{NewIdea($A$, $\Omega$, $R$, $\gamma$)}.
%\end{itemize}
%\Ensure $C_1^{\#}$.
%\end{algorithmic}
%\end{algorithm}

\begin{theorem}
Suppose $T\subset C_1$. Then
\begin{equation}
    \frac{|C_1\triangle C_1^{\#}|}{|C_1|}\leq o(1)
\end{equation}
\end{theorem}

\begin{proof}
It is equivalent to show $|C_1\triangle C_1^\#|\leq o(n_1)$.
Note that $\mathbf{x}^*=\mathbf{1}_{C_1\setminus T}$. By Theorem \ref{IndAna}, we get $\|\mathbf{1}_{C_1\setminus T}-\mathbf{x}^\#\|_2\leq o(\|\mathbf{1}_{C_1\setminus T}\|_2)=o(\sqrt{n_1})$. We then apply Lemma \ref{indicatornorm} with $K=C_1\setminus T$, $W^\#=C_1^\#$, and $\mathbf{v}=\mathbf{x}^\#$ to get $|(C_1\setminus T)\triangle C_1^\#|\leq o(n_1)$. Therefore $|C_1\triangle C_1^\#|\leq o(n_1)$.
\end{proof}

\begin{figure}[t]
    \centering
    %\resizebox{\columnwidth}{!}{
    % \begin{tabular}{cc}
\includegraphics[width=0.73\linewidth]{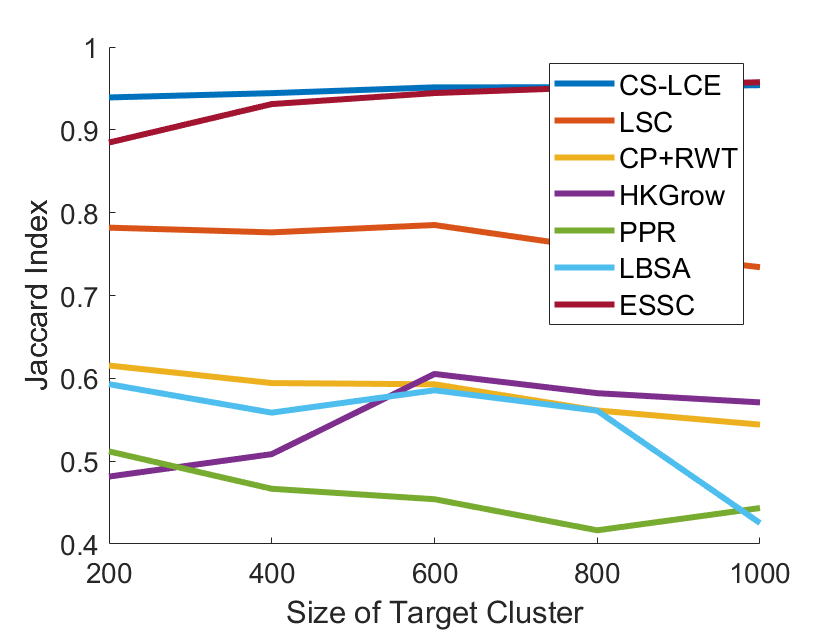} 
     	\includegraphics[width=0.73\linewidth]{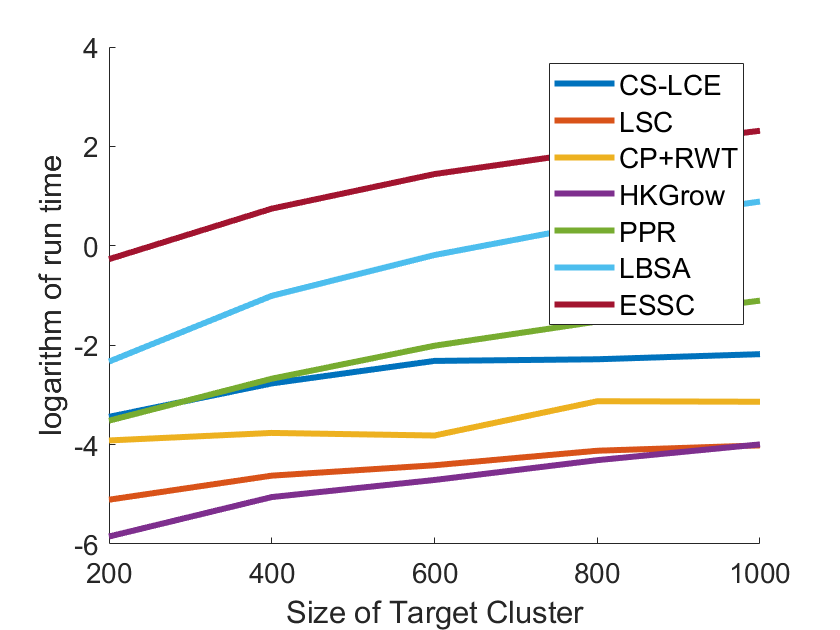}
    % \end{tabular}
    %}
    \vspace{-2mm}
    % \caption{Average Jaccard Index on SSBM.}
    \caption{Performances on Symmetric Stochastic Block Model. \emph{Top}: Average Jaccard Index. \emph{Bottom}: Logarithm of Average Run Time.}
    
    \vspace{-2mm}
	\label{SSBMresults}
\end{figure} 

\begin{figure}[t]
    \centering
    % \resizebox{\columnwidth}{!}{
    % \begin{tabular}{cc}
    \includegraphics[width=0.75\linewidth]{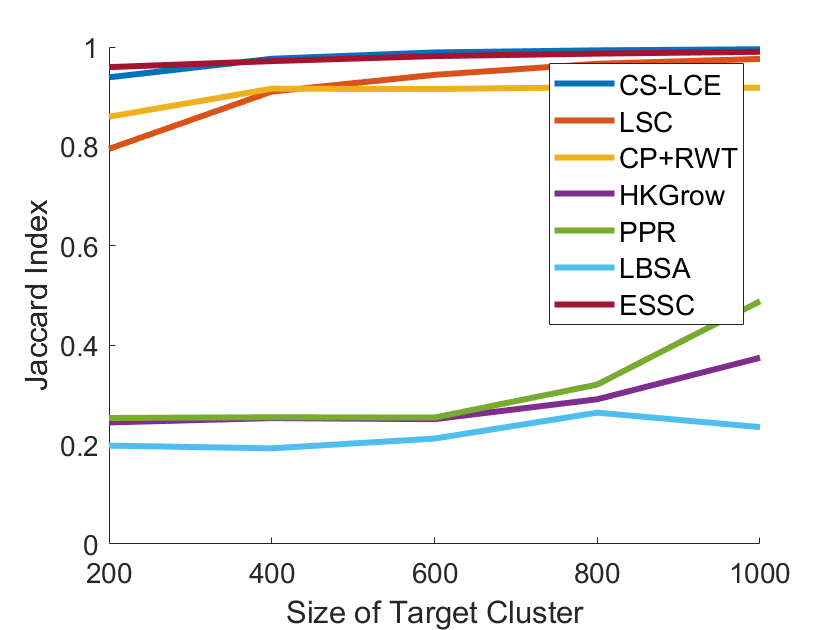} 
    \includegraphics[width=0.75\linewidth]{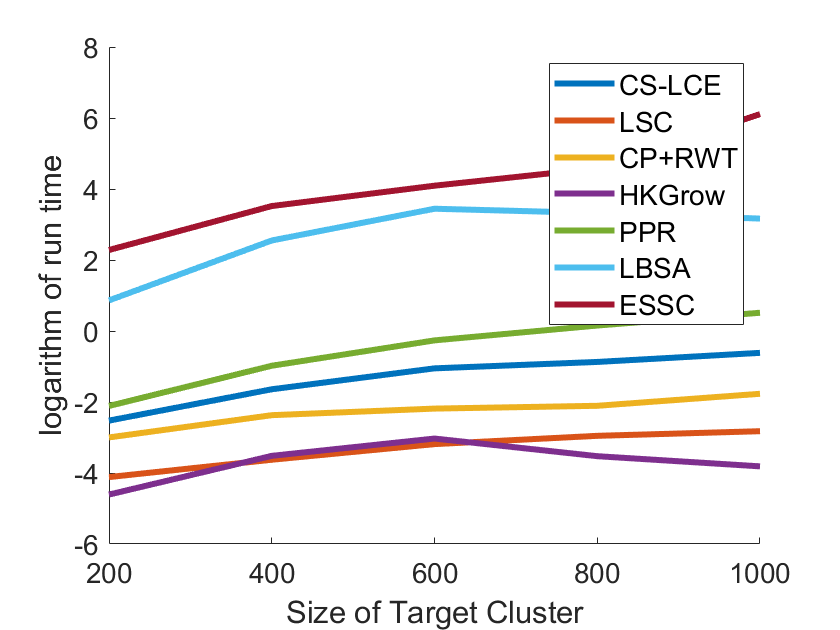}
    % \end{tabular}
    %\caption{\emph{Left}: Average Jaccard Index. \emph{Right}: Logarithm of the Average Run Time.}
    %\vspace{-2mm}
    % }
    % \caption{Average Jaccard Index on SBM.}
    \vspace{-2mm}
    \caption{Performances on Non-symmetric Stochastic Block Model. \emph{Top}: Average Jaccard Index. \emph{Bottom}: Logarithm of Average Run Time.}
	\label{SBMresults}
    \vspace{-2mm}
\end{figure}  

\begin{figure}[t]
    \centering
    \resizebox{\columnwidth}{!}{
    \begin{tabular}{ccc}
	\includegraphics[width=0.2\textwidth]{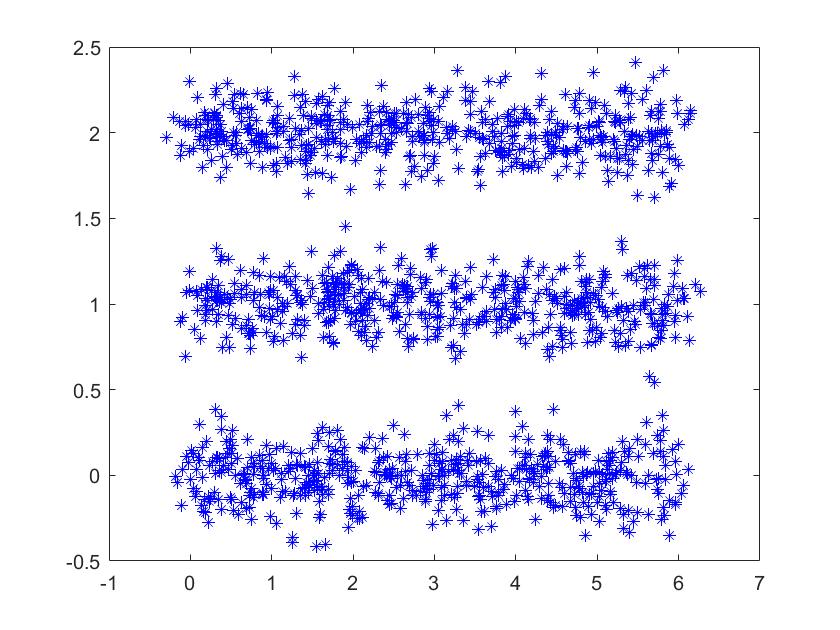}  & \includegraphics[width=0.2\textwidth]{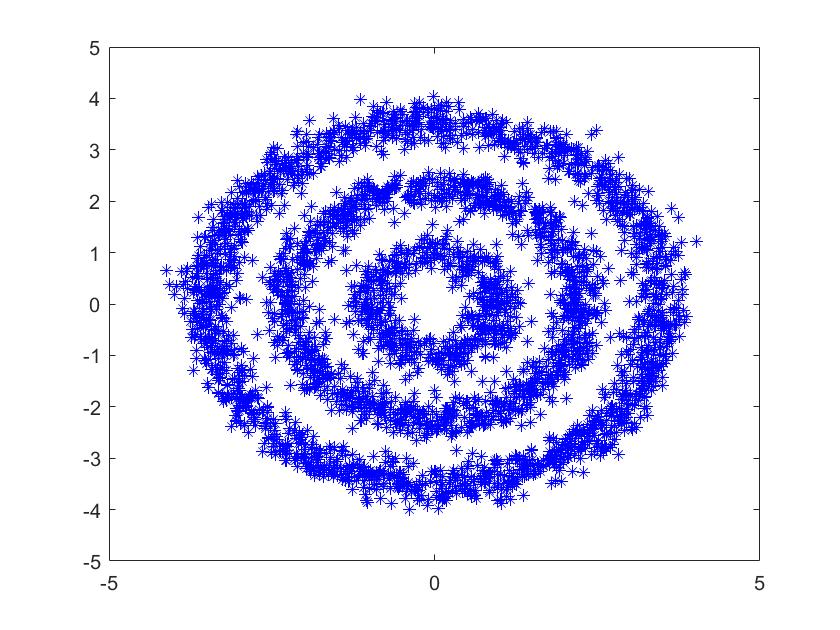} & \includegraphics[width=0.2\textwidth]{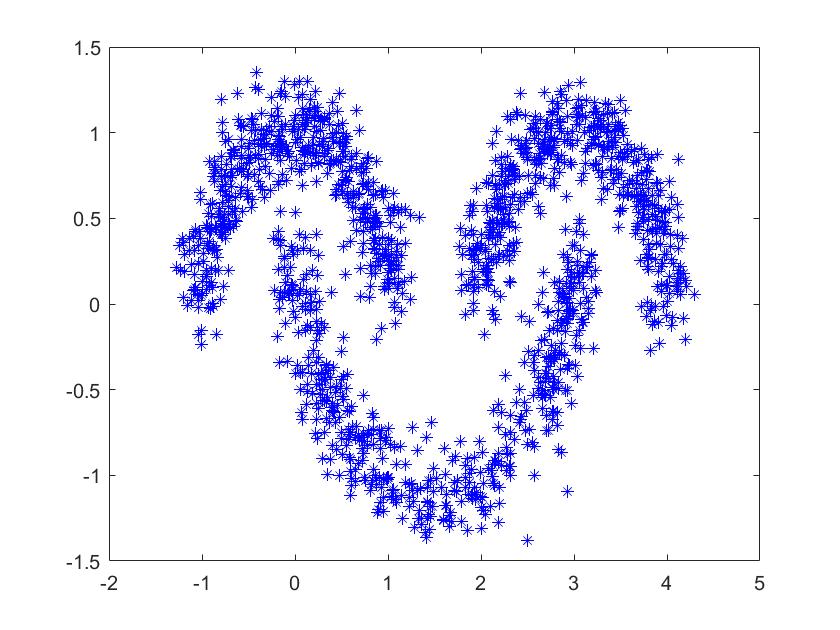} 
	\end{tabular} 
 }
    \vspace{-2mm}
    \caption{Visualizations of Geometric Data. From \emph{Left} to \emph{Right}: Three Lines, Three Circles, and Three Moons.}
    \vspace{-3mm}
    \label{GeoData}
\end{figure}

% \section{Computational Complexity}
% The computational complexity is $O(nd_{\max}\log n)$.

\section{Experiments} \label{secExp}
In this section, we evaluate Algorithm \ref{alg2} on various synthetic and real datasets and compare its performance with several baselines. For all experiments, we perform 100 individual runs. Additional details about the experiments are provided in the supplement. We make the supplement and code available at: \url{https://github.com/zzzzms/LocalClustering}.

\paragraph{Datasets.} We use simulated stochastic block model, simulated geometric data with three particular shapes, network data on political blogs\cite{AG05}, OptDigits\footnote{\url{https://archive.ics.uci.edu/ml/datasets/optical+recognition+of+handwritten+digits}}, AT\&T Database of Faces\footnote{\url{https://git-disl.github.io/GTDLBench/datasets/att_face_dataset/}}, MNIST\footnote{\url{http://yann.lecun.com/exdb/mnist/}}, and USPS\footnote{\url{https://git-disl.github.io/GTDLBench/datasets/usps_dataset/}} as our benchmark datasets. 

\paragraph{Baselines and Settings.} We adopt the LSC~\cite{LaiShen2022}, CP+RWT~\cite{LaiMckenzie2020}, HK-Grow~\cite{Kloster2014}, PPR~\cite{Andersen2007}, ESSC~\cite{Wilson2014}, LBSA~\cite{Shi2019}, and several other modern semi-supervised clustering algorithms as our baseline methods. For our experiments of stochastic block model, the only target cluster is the most dominant cluster, i.e., the cluster with the highest connection probability. For all other experiments, all of the clusters are considered as our target clusters, and we apply CS-LCE iteratively to extract all of them. We use Jaccard index to measure the performance of one cluster tasks and use mean accuracy across all clusters to measure the performance of multiple clusters tasks. 
% \subsection{Stochastic Block Model Data}
% The stochastic block model \cite{Holland1983} is a natural generalization of the classic Erd\"os-R\'enyi model \cite{ER1959}. It is a generative model for random graphs 
% with certain edge densities within and between underlying clusters, such that the edges within 
% clusters are more dense than the edges between clusters. In our experiments, we will consider both symmetric and non-symmetric stochastic block model, and we are only interested in recovering a single target cluster.

% In our experiment, we choose the coordinates of the center of Three Lines as XX,XX,XX; 
% we choose the ambient dimension equals to 100 for all cases, and adding a Gaussian random noise.
% \begin{table}[t]
% 	%\vspace{-3mm}
% 	\centering
% 	\begin{tabular}{c|ccc}
% 	\toprule
%            Datasets & CS-LCE & LSC & CP+RWT  \\
%     \midrule
%     3 Lines & \textbf{94.9} & 92.9 & 91.3  \\
% 	3 Circles & \textbf{95.4} & 94.0 & 91.3  \\
% 	3 Moons & \textbf{97.4} & 94.4 & 96.8  \\
% 	\bottomrule
% 	\end{tabular}
% \caption{Clustering Accuracy on Geometric Data (\%).}
% \label{GeoTable}
% \end{table}    

\begin{table}[t]
	%\vspace{-3mm}
	\centering
	\begin{tabular}{c|ccc}
	\toprule
           Datasets & 3 Lines & 3 Circles & 3 Moons  \\
    \midrule
    
	LSC & 89.0 (5.53) & 96.2 (3.71) & 85.3 (1.88) \\
	CP+RWT & 82.1 (9.06) & 96.1 (5.09) & 85.4 (1.33)  \\
 \textbf{CS-LCE} & \textbf{92.4} (8.13) & \textbf{97.6} (4.69) & \textbf{96.8} (0.89)  \\
    %SC & Fail & Fail & Fail \\
	\bottomrule
	\end{tabular}
 \vspace{-2mm}
\caption{Mean Accuracy and SD on Geometric Data (\%)}
 \vspace{-2mm}
\label{GeoTable}
\end{table}    

\begin{table}[t]
    \centering
	\begin{tabular}{c|ccc}
		\toprule
		Label Ratios   & 10 \% & 20 \% & 30 \%  \\
		\midrule
		
		LSC & $94.8$ (3.32) & $97.8$ (1.18) & $98.2$ (0.77)
		\\
		CP+RWT & $93.7$ (3.34) & $97.8$ (1.44) & $98.3$ (0.43)
		\\
		SC & 95.8 (0.00) & 95.8 (0.00) & 95.8 (0.00) \\
  \textbf{CS-LCE} & \textbf{98.0} (1.90) & \textbf{99.1} (0.79) & \textbf{99.3} (0.59)
		\\
		\bottomrule
	\end{tabular}
  \vspace{-2mm}
\caption{Mean Accuracy and SD on AT\&T Data (\%)}
 \vspace{-2mm}
\label{TableATT}
\end{table}  

\begin{figure}[t]	
	\centering
        %\resizebox{0.9\columnwidth}{!}{
	%\begin{tabular}{cc}	
 \includegraphics[width=0.2\textwidth]{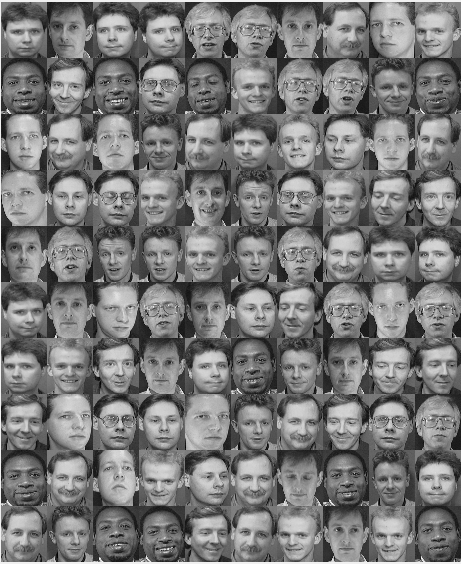} 
		\includegraphics[width=0.2\textwidth]{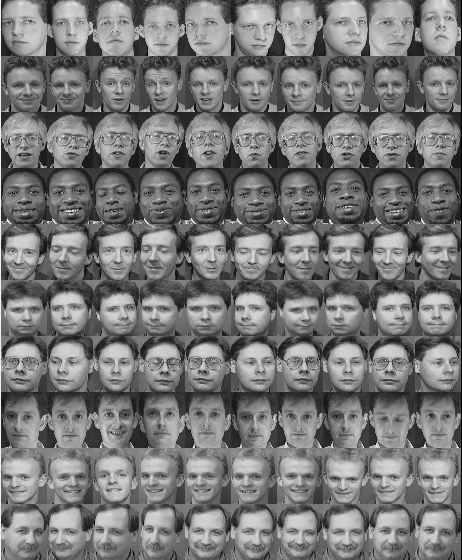}
	%\end{tabular}
 %}
  \vspace{-1mm}
	\caption{\emph{Left}: Randomly Permuted AT\&T Faces. \emph{Right}: Desired Recovery of all Clusters.}
	\label{FigureATT}
  \vspace{-1mm}
\end{figure}

\begin{figure}[t]
    \centering
    % \resizebox{\columnwidth}{!}{
    % \begin{tabular}{c}
    \includegraphics[width=0.75\linewidth]{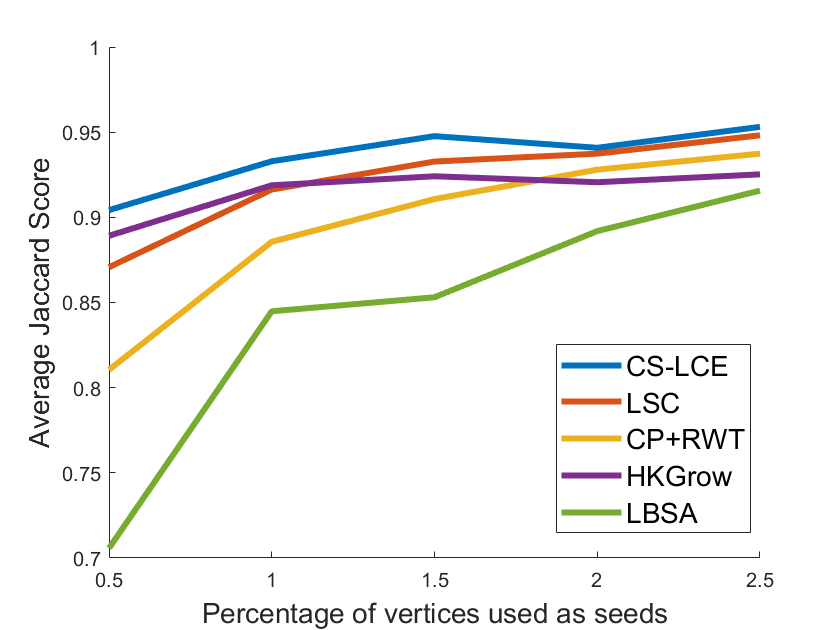} 
    % & 	\includegraphics[width=0.42\textwidth]{CS_LSC_CPRWT_Others_Time_Optdigits.png}
    % \end{tabular}
    %\caption{\emph{Left}: Average Jaccard Index. \emph{Right}: Logarithm of the Average Run Time.}
    %\vspace{-2mm}
    % }
     \vspace{-1mm}
    \caption{Average Jaccard Index on OptDigits.}
	\label{OptDigitresults}
    \vspace{-1mm}
 
\end{figure} 

\subsection{Simulated Data}
% In this part, we will test our algorithm on the simulated stochastic block model and simulated geometric data with certain shapes.
\paragraph{Symmetric Stochastic Block Model.}
%\subsubsection{Symmetric Stochastic Block Model}
The stochastic block model is a generative model for random graphs with certain edge densities within and between underlying clusters. The edges within clusters are denser than the edges between clusters. In the case of each cluster has the same size and the intra- and inter-connection probability are the same among all vertices, we have the symmetric stochastic block model $SSBM(n,k,p,q)$. The parameter $n$ is the size of the graph, $k$ is the number of clusters, $p$ is the probability of intra-connectivity, and $q$ is the probability of inter-connectivity. In our experiments, we fix $k=3$ and vary $n$ among $600,1200,1800,2400,3000$. We choose $p=5\log n/n, q=\log n/n$. With five labeled vertices as seeds, we achieve the performances shown in Figure~\ref{SSBMresults}. We can see CS-LCE outperforms all other baselines with a reasonable running time.
%especially when the size of target cluster is relatively small.

%\subsubsection{Non-symmetric Stochastic Block Model}
\paragraph{Non-symmetric Stochastic Block Model.}
In a more general stochastic block model $SBM(\mathbf{n},k,P)$, where $n$ and $k$ are the same as symmetric case. The matrix $P$ indicates the connection probability within each individual cluster and between different clusters. 
It is worthwhile to 
note that the information theoretical bound for exact cluster recovery in SBM are given in \cite{Abbe2018} and \cite{Abbe2015}. 
In our experiments, we fix $k=3$, and the size of clusters are chosen as $\mathbf{n}=(n_1, 2n_1, 5n_1)$ where $n_1$ is chosen from $\{200, 400, 600, 800, 1000\}$. We set the connection probability matrix $P=[p,q,q; q,p,q; q,q,p]$ where $p=\log^2(8n_1)/(8n_1)$ and $q=5\log(8n_1)/(8n_1)$. With five labeled vertices as seeds, the clustering performances are shown in Figure~\ref{SBMresults}.

\begin{table}[t]
	\centering
	\begin{tabular}{c|ccc}
		\toprule
		Label Ratios & 0.05 \% & 0.10 \%  & 0.15 \% \\
		\midrule
		
		LSC & 77.0 (3.47)  & 83.6 (2.76) & 88.8 (2.52) \\
		CP+RWT & {74.1} (3.13) & 79.7 (2.43) & 85.0 (2.37) \\
  \textbf{CS-LCE} & \textbf{85.3} (2.67) & \textbf{89.8} (1.91) & \textbf{93.2} (1.76) \\
	    \bottomrule
	\end{tabular}
  \vspace{-1mm}
        \caption{Mean Accuracy and SD on MNIST (\%)}
         \vspace{-1mm}
        \label{tableMNIST}
\end{table}   

% \begin{table}[t]
% 	\centering
% 	\resizebox{\columnwidth}{!}{%
% 	\begin{tabular}{c|ccc}
% 		\toprule
% 		Label Ratio & CS-LCE & LSC \cite{LaiShen2022} & CP+RWT \cite{LaiMckenzie2020} \\
% 		\midrule
% 		0.2 \% & \textbf{78.6} & 74.4 & 60.0 \\
% 		0.3 \% & \textbf{79.9} & 76.7 & 65.1 \\
% 		0.4 \% & \textbf{84.4} & 80.6 & 64.6 \\
% 	    \bottomrule
% 	\end{tabular}
% 	}
%         \caption{Clustering Accuracy on USPS Data (\%).}
%         \label{tableUSPS}
% \end{table}   

\begin{table}[t]
	\centering
	
	\begin{tabular}{c|ccc}
        % \backslashbox{Methods}{Label Ratios}
		\toprule
	Label Ratios	& 0.2 \% & 0.3\% & 0.4\% \\
		\midrule
		
		LSC & 72.3 (3.54) & 77.1 (3.42) & 80.4 (3.20) \\
		CP+RWT & 68.9 (3.17) & 73.3 (2.76) & 76.6 (2.59) \\
  \textbf{CS-LCE} & \textbf{76.8} (3.37) & \textbf{80.1} (3.14) & \textbf{84.1} (2.53) \\
	    \bottomrule
	\end{tabular}
  \vspace{-2mm}
\caption{Mean Accuracy and SD on USPS (\%)}
 \vspace{-2mm}
\label{tableUSPS}
\end{table}

\begin{table}[!t]
	\centering
	\begin{tabular}{c|cc}
		\toprule
		  & MNIST & USPS \\
		\midrule
		
	    KM-cst~\cite{Basu2004} & 54.27  & 68.18   \\
		AE+KM~\cite{MacQueen1967} & 74.09 & 70.28  \\
		 AE+KM-cst~\cite{Basu2004} & 75.98 & 71.87  \\
		DEC~\cite{Xie2016} & 84.94 & 75.81  \\
        IDEC~\cite{Guo2017} & 83.85 & 75.86 \\
        SDEC~\cite{Ren2019} & 86.11 & 76.39 \\
        \textbf{CS-LCE} (Ours) & \textbf{96.02}  & \textbf{82.10}   \\
	    \bottomrule
	\end{tabular}
  \vspace{-1mm}
 \caption{Mean Accuracy on MNIST and USPS (\%)}
  \vspace{-2mm}
 \label{tableOthers}
\end{table}

%\subsubsection{Geometric Data}
\paragraph{Geometric Data.}
We also simulated three high dimensional datasets in Euclidean space where the projections of the clusters onto two dimensional plane look like three lines, three circles, or three moons. See Figure~\ref{GeoData} for an illustration of them. These datasets are often used as benchmark for data clustering and they are also described in \cite{Mckenzie2019} with slightly different parameters. Because of the shape of underlying clusters, traditional $k$-means clustering or spectral clustering fail on these contrived datasets. In our experiments, for each dataset, we randomly select 10 seeds for each of the cluster. The mean accuracy and standard deviation of CS-LCE compared with LSC \cite{LaiShen2022} and CP+RWT \cite{LaiMckenzie2020} are given in Table \ref{GeoTable}. A more detailed description of this simulated dataset is given in the supplement.

\subsection{Human Face Images}
The AT\&T Database of Faces 
contains gray-scale images for $40$ different people of pixel size $92\times 112$. Images of each person are taken under $10$ different conditions, by varying the three perspectives of faces, lighting conditions, and facial expressions. We use part of this dataset by randomly selecting 10 people such that each individual is associated with 10 pictures of themselves. The selected dataset and desired recovery are shown in Figure~\ref{FigureATT}. 

The mean accuracy and standard deviation of CS-LCE compared with LSC \cite{LaiShen2022}, CP+RWT \cite{LaiMckenzie2020}, and spectral clustering (SC) are summarized in Table~\ref{TableATT}. Note that spectral clustering method is unsupervised, hence its accuracy does not affected by the label ratios.

\subsection{Network Data}
``The political blogosphere and the 2004 US Election" \cite{AG05} dataset contains a list of political blogs that were classified as liberal or conservative with links between blogs. An illustration of this dataset is attached in the supplement. The state-of-the-art result on this dataset is given in \cite{Abbe2015}. Their simplified algorithm gave a successful classification 37 times out of 40 trials, and each of the successful trials correctly classified all but 56 to 67 of the 1,222 vertices in the graph main component. 

In our experiments, given one labeled seed, CS-LCE succeeds 35 trials out of a total of 40 trials. Among these 35 successful trials, the average number of misclassified node in the graph main component is 49, which is comparable to the state-of-the-art result. We note that LSC \cite{LaiShen2022} also succeeds 35 out of 40 trials, but the average number of misclassified node equals to 55. We also note that CP+RWT \cite{LaiMckenzie2020} fails on this dataset.

\subsection{Digits Data}
\paragraph{OptDigits.}
%\subsubsection{OptDigits}
This dataset contains grayscale images of handwritten digits from $0$ to $9$ of size $8\times 8$. There are a total of $5,620$ images and each cluster has approximately $560$ images. The average Jaccard index of CS-LCE compared with several other algorithms are shown in Figure~\ref{OptDigitresults}. we exclude PPR and ESSC in the comparison as they either too slow to run or the accuracy is too low.

%\subsubsection{MNIST and USPS}
\paragraph{MNIST and USPS.}
The MNIST dataset consists of $70,000$ grayscale images of the handwritten digits $0$-$9$ of size $28\times 28$ with approximately $7,000$ images of each digit. The USPS data set contains 9298 grayscale images, obtained from the scanning of handwritten digits from envelopes by the U.S. postal service. 
We test CS-LCE, LCS, CP+RWT, and several other modern semi-supervised methods on these two datasets, the results are show in Table~\ref{tableMNIST}, Table~\ref{tableUSPS} and Table~\ref{tableOthers}. 
It is worth pointing out that in Table~\ref{tableMNIST} and Table~\ref{tableUSPS}, we have only very few labeled data for our tasks. If one uses a neural network method to train for classification of images, then it usually needs more labeled data for training. 
In Table \ref{tableOthers}, we compare CS-LCE with several other constraint clustering algorithms. In each constrained clustering algorithms, the total number of pairwise constraints are set to equal to the total data points. Therefore in order to have a fair comparison, we choose a certain amount of labeled data in CS-LCE such that the total pairwise constraints are the same. 
%In Table \ref{tableMNIST} and \ref{tableUSPS}, the percentages are the average accuracy of recovering all the clusters. The run time is the average run time for recovering one cluster and are measured in seconds.

\section{Conclusions} \label{secConclusion}
In this work, we proposed a semi-supervised local clustering approach based on compressive sensing. Our approach improves the disadvantages in prior work under the same framework, and it is shown to be asymptotically correct under certain assumptions of graph structure. Extensive Experiments on various datasets have validated its effectiveness. We hope this work will draw people's interests and bring attentions to this new perspective of local clustering. Potential research directions in the future could be done on developing a more calibrated way of choosing the removal set and investigating how to incorporate compressive sensing into some modern architectures such as deep neural networks.

\section*{Acknowledgements}
The second author is supported by the Simon Foundation Collaboration Grant \#864439. 
The third author is supported by the U.S.~Army Research Office Award under Grant Number W911NF-21-1-0109. 
We want to thank all the reviewers for their valuable feedback to improve the quality of this paper.

\bibliographystyle{named}
\bibliography{ijcai23}

\end{document}